\documentclass[11pt]{article}

\usepackage[english]{babel}

\usepackage[letterpaper,top=2cm,bottom=2cm,left=3cm,right=3cm,marginparwidth=1.75cm]{geometry}

\usepackage{amsfonts,amssymb,amsmath,amscd}
\usepackage{amsmath}
\usepackage{floatrow}
\usepackage{graphicx,wrapfig,lipsum}
\usepackage{amssymb}
\usepackage{graphicx}
\usepackage{subcaption}
\usepackage{float}
\usepackage{mathrsfs}
\usepackage{booktabs,tabularx}
\usepackage{latexsym}
\usepackage{booktabs,siunitx}
\usepackage{enumerate}
\usepackage{booktabs}
\usepackage{url}
\usepackage{listings}
\usepackage{color}
\usepackage{hyperref}
\usepackage{ulem}
\usepackage{array, makecell}
\usepackage{sidecap}
\usepackage{amsfonts}
\usepackage{verbatim}
\usepackage{amsthm}
\usepackage{dsfont}
\usepackage{url}
\usepackage{tocloft}
\usepackage{float} 
\usepackage{ae}
\usepackage{xcolor}
\usepackage{amsfonts}
\usepackage{amsmath}
\usepackage{bm}
\usepackage{mathtools}
\usepackage{dsfont}
\usepackage{graphicx}
\usepackage{soul}
\usepackage{titlesec}
\usepackage[utf8]{inputenc}
\usepackage{amsthm}
  
\usepackage{sectsty}
\usepackage{enumitem}
\usepackage{titletoc}
\usepackage{xr}
\newtheorem{theorem}{Theorem}[section]

\newtheorem{corollary}[theorem]{Corollary}
\numberwithin{equation}{section}

\title{On the Reliability of Generative Augmentation:\\ A Wasserstein-Based Theoretical and Empirical Study}
\author{
Chathurika S Abeykoon\\
abeykoonc@rhodes.edu\\
  \and
Mathias Nthiani Muia\\
mnmuia@southalabama.edu\\
\and
  Mallory Goldstein\\
golmd-27@rhodes.edu \\
   }
\date{}
\begin{document}
\maketitle
\begin{abstract}

Generative data augmentation is widely used to mitigate class imbalance, yet its theoretical effect on downstream generalization remains poorly understood. In this work, we develop a statistical framework for conditional generative augmentation and analyze its impact on classification risk. We formalize augmentation as a distribution-mixing process and show that the resulting risk distortion is controlled by both the augmentation strength and the class-conditional Wasserstein discrepancy between real and generated distributions. We further derive a capacity-dependent generalization bound based on Rademacher complexity, revealing an explicit trade-off between hypothesis complexity, augmentation intensity, and generative fidelity.

Empirically, we evaluate the framework on binary and multiclass imbalanced classification tasks using Conditional GAN  and Conditional WGAN-GP augmentation. Across datasets, CWGAN-GP consistently achieves lower Wasserstein discrepancies than CGAN, indicating improved distributional fidelity. However, improved fidelity does not necessarily translate into superior classification performance, with classical oversampling methods often remaining competitive. These findings support the central theoretical prediction that augmentation reliability is governed by distributional approximation error rather than predictive performance alone.

Overall, this work establishes generative augmentation as a distributional perturbation process whose reliability can be quantified through Wasserstein-based measures and supported by finite-sample generalization guarantees. The proposed framework provides a principled foundation for evaluating synthetic data quality beyond classification accuracy alone.

\end{abstract}


%-----------------------------------------------------------------------------------------------------------------
\section{Introduction}
%-------------------------------------------------------------------------------------------------------------------
%-------------------------------------------------------------------------------------------------------------------

Accurate classification of structured data is a central problem in machine learning. In practice, many classification tasks are characterized by limited labeled data and imbalanced class distributions, where certain classes are substantially underrepresented. Such imbalance distorts the empirical distribution used for learning, often leading to biased decision boundaries, reduced minority-class recall, and degraded generalization performance.

From a statistical learning perspective, class imbalance alters the empirical distribution used for risk minimization as discussed by He et al. in \cite{HeGarcia2009}. Let $(X,Y) \sim p_{\text{data}}(x,y)$ denote the joint distribution of features and class labels . In finite samples, empirical risk minimization approximates the population risk under the empirical distribution $\hat p(x,y)$. When $\hat p(y)$ is highly skewed, the learned classifier tends to prioritize majority classes, implicitly minimizing risk under a distorted sampling distribution. This imbalance can lead to biased decision boundaries and poor minority-class recall, even when overall accuracy appears high. Conventional mitigation strategies, including class reweighting, undersampling, and oversampling, partially address this issue but often fail to capture the underlying class-conditional feature distributions in a principled manner.

Synthetic data augmentation provides a promising alternative. Rather than merely replicating minority samples or adjusting loss weights, generative models aim to approximate the conditional distribution $p_{\text{data}}(x \mid y)$ and sample new feature vectors from the learned distribution. As a result, synthetic augmentation can enrich the effective support of minority classes and encourage classifiers to learn smoother and more representative decision boundaries. However, applying generative augmentation to structured feature spaces introduces several practical challenges. In contrast to image-based settings, where convolutional inductive biases help preserve spatial coherence and support realistic sample generation, tabular feature vectors typically lack inherent structure. As a result, statistical relationships among variables must be learned directly from often limited data, increasing the difficulty of accurately modeling the underlying distribution. Under these conditions, generative models are more susceptible to training instability, distributional artifacts, noise amplification, and mode collapse, all of which can negatively impact downstream classification performance rather than improve it.\cite{Lucic2018}

Generative Adversarial Networks (GAN) offers a flexible framework for learning complex, high-dimensional data distributions through adversarial training \cite{Goodfellow2014}. A generator network maps latent noise variables to synthetic samples, while a discriminator (or critic) attempts to distinguish real from generated data. Conditional GAN (CGAN) extends this framework by conditioning both generator and discriminator on class labels, enabling targeted generation of class-specific samples. In settings with class imbalance, conditional generation allows focused augmentation of underrepresented varieties. Nevertheless, classical GAN training based on Jensen--Shannon divergence is known to suffer from vanishing gradients and unstable dynamics when real and generated distributions have limited support overlap.

To address these challenges, Wasserstein GANs (WGANs) reformulate the adversarial objective using the Wasserstein-1 distance, a metric grounded in optimal transport theory that provides informative gradients even when the supports of the real and generated distributions do not overlap \cite{Arjovsky2017}. Within this framework, the discriminator is replaced by a Lipschitz-constrained critic trained to approximate the Wasserstein distance between real and synthetic distributions, resulting in a smoother optimization landscape and more stable convergence behavior relative to divergence-based objectives \cite{Arjovsky2017,Gulrajani2017}. The gradient penalty formulation (WGAN-GP) further enforces the Lipschitz constraint without restricting model capacity through weight clipping, improving training stability, sample diversity, and robustness to mode collapse across a variety of applications \cite{Gulrajani2017,Lucic2018}. These theoretical and practical advantages make WGAN-GP particularly well suited for modeling continuous feature distributions and for downstream learning tasks that are sensitive to distributional mismatch.

While generative augmentation has demonstrated empirical success in addressing class imbalance, its theoretical implications for classification performance remain insufficiently understood. Existing studies \cite{DouzasBacao2018,Lucic2018} largely evaluate augmentation strategies through predictive accuracy, with limited attention to how discrepancies between real and generated distributions affect downstream learning. Consequently, many approaches lack formal guarantees regarding the reliability of synthetic data augmentation.

In this work, we develop a theoretical and empirical framework for studying the reliability of conditional generative augmentation in imbalanced classification. We employ CGAN and Conditional WGAN-GP (CWGAN-GP) architectures to learn class-conditional feature distributions and generate synthetic samples for data augmentation. Building on concepts from optimal transport, we derive a Wasserstein-based risk bound that explicitly quantifies how discrepancies between the learned and true data distributions influence classification risk. The framework is evaluated on both binary and multiclass classification tasks exhibiting natural class imbalance.

Experimental results demonstrate that CWGAN-GP consistently achieves lower Wasserstein discrepancies than CGAN, indicating improved distributional fidelity of the generated samples. However, improved fidelity does not necessarily translate into superior classification performance relative to classical oversampling methods. These findings highlight an important distinction between synthetic data fidelity and predictive utility and motivate the need for principled reliability measures beyond classification accuracy alone. Overall, the proposed framework provides both theoretical guarantees and practical tools for evaluating the reliability of generative augmentation in imbalanced learning settings.

\section{Related Work}

\subsection{Class Imbalance and Data Augmentation}

Class imbalance is a longstanding challenge in supervised learning because empirical risk minimization tends to favor majority classes, often resulting in poor minority-class performance \cite{HeGarcia2009}. Traditional approaches include class reweighting, undersampling, and oversampling techniques that modify the effective training distribution. Among these methods, the Synthetic Minority Oversampling Technique (SMOTE) generates additional minority-class samples through interpolation between neighboring observations and remains one of the most widely used approaches for imbalanced classification \cite{Chawla2002}. Although such methods can improve predictive performance, they primarily alter sample frequencies and do not explicitly model the underlying class-conditional distributions.

Recent advances in generative modeling have motivated the use of synthetic data augmentation as an alternative strategy for addressing class imbalance. Rather than replicating or interpolating existing observations, generative models attempt to learn the data-generating process and sample new observations from the learned distribution. This perspective has led to growing interest in generative augmentation across a variety of application domains.

\subsection{Generative Models and Wasserstein-Based Learning}

Generative Adversarial Networks (GANs) introduced a powerful framework for learning complex data distributions through adversarial training \cite{Goodfellow2014}. Conditional GANs (CGANs) extend this framework by incorporating class labels into both the generator and discriminator, enabling class-specific sample generation and targeted augmentation of minority classes \cite{Mirza2014}. Numerous studies have reported improved classification performance through GAN-based augmentation in both image and structured-data settings \cite{DouzasBacao2018}.

Despite their success, classical GANs are known to suffer from training instability, mode collapse, and sensitivity to divergence-based objectives \cite{Lucic2018}. Wasserstein GANs address these issues by replacing divergence-based training objectives with the Wasserstein-1 distance, a metric grounded in optimal transport theory \cite{Arjovsky2017}. The gradient-penalty formulation (WGAN-GP) further improves stability by enforcing the Lipschitz constraint through gradient regularization rather than weight clipping \cite{Gulrajani2017}. These developments have made Wasserstein-based generative models particularly attractive for applications where distributional fidelity is important.

\subsection{Generative Augmentation and Distributional Reliability}

Most existing studies evaluate generative augmentation primarily through downstream predictive performance, such as accuracy, F1 score, or recall \cite{DouzasBacao2018,FridAdar2018}. While these metrics provide useful information regarding task-specific performance, they offer limited insight into how closely the generated samples approximate the underlying data-generating distribution. As a result, the relationship between generative fidelity and classification performance remains poorly understood.

Several recent works have advocated the use of distributional metrics for evaluating synthetic data quality, including measures based on optimal transport and Wasserstein distances. However, formal connections between distributional approximation error and downstream classification risk remain relatively underdeveloped. In particular, few studies provide theoretical guarantees describing how discrepancies between real and generated distributions influence the reliability of synthetic augmentation.

The present work addresses this gap by combining conditional generative augmentation with a Wasserstein-based theoretical framework. Rather than evaluating synthetic data solely through predictive performance, we analyze augmentation as a distribution-mixing process and derive explicit risk bounds linking class-conditional Wasserstein approximation error to classification risk. This perspective provides a principled framework for assessing the reliability of synthetic data augmentation in imbalanced learning settings.

%-----------------------------------------------------------------------------------------------------------------
\subsection{Our Contribution}
%-----------------------------------------------------------------------------------------------------------------

The main contributions of this work are summarized as follows:

\begin{itemize}

\item \textbf{A Wasserstein-Based Framework for Augmentation Reliability.}
We develop a theoretical framework for analyzing the reliability of generative data augmentation in imbalanced classification. The framework explicitly links class-conditional distributional discrepancies between real and synthetic data to downstream classification risk.

\item \textbf{Risk Bounds for Conditional Generative Augmentation.}
We derive a Wasserstein-based risk bound showing that augmentation-induced risk distortion is controlled by both the augmentation strength and the class-conditional Wasserstein discrepancy between the true and generated distributions. We further establish capacity-dependent generalization guarantees through a Rademacher complexity analysis.

\item \textbf{Distributional Fidelity Evaluation.}
We introduce class-wise and weighted Wasserstein fidelity measures that provide interpretable and theoretically motivated criteria for assessing the quality and reliability of synthetic data generated by conditional generative models.

\item \textbf{Empirical Validation on Imbalanced Classification Tasks.}
We conduct extensive experiments on binary and multiclass datasets exhibiting natural class imbalance. The empirical results demonstrate that CWGAN-GP consistently achieves lower class-conditional Wasserstein discrepancies than CGAN, supporting the theoretical relationship between generative fidelity and augmentation reliability while highlighting the distinction between distributional fidelity and downstream predictive performance.

\end{itemize}

%-----------------------------------------------------------------------------------------------------------------
\section{Theoretical Framework for Generative Augmentation Reliability}
%-----------------------------------------------------------------------------------------------------------------

This section develops a statistical learning framework for analyzing conditional generative augmentation in structured feature space. Our goal is to formally characterize how the fidelity of synthetic data generated by a conditional adversarial model influences downstream classification risk and augmentation reliability. Rather than treating augmentation as a heuristic data-expansion procedure, we interpret it as controlled distributional mixing and quantify its impact using tools from optimal transport theory which is an alternative way of measuring the distance between probability distributions.

%-----------------------------------------------------------------------------------------------------------------
\subsection{Learning Under Imbalanced Sampling}
%-----------------------------------------------------------------------------------------------------------------

Let $(X,Y)\sim P$ denote the true joint distribution over the feature space $\mathcal{X}\subset\mathbb{R}^d$ and label set $\mathcal{Y}=\{1,\ldots,K\}$. For a classifier $f:\mathcal{X}\rightarrow\mathcal{Y}$ and loss function $\ell$, the population risk is defined as
\begin{equation}
R_P(f)
=
\mathbb{E}_{(x,y)\sim P}
\left[\ell(f(x),y)\right].
\end{equation}
Given a training sample
$\mathcal{D}=\{(x_i,y_i)\}_{i=1}^{n}$,
learning proceeds by minimizing the empirical risk
\begin{equation}
\widehat{R}_n(f)
=
\frac{1}{n}
\sum_{i=1}^{n}
\ell(f(x_i),y_i).
\end{equation}

Under class imbalance, empirical risk minimization tends to emphasize majority classes while underrepresenting minority-class risk. Generative augmentation seeks to alleviate this issue by learning class-conditional distributions and generating additional minority-class samples.

%-----------------------------------------------------------------------------------------------------------------
\subsection{Conditional Generative Modeling}
%-----------------------------------------------------------------------------------------------------------------

For each class $y\in\mathcal{Y}$, let
\begin{equation}
P_y = P(\cdot \mid Y=y)
\end{equation}
denote the true class-conditional distribution. A conditional generator is a measurable mapping
\begin{equation}
G_{\theta}:\mathcal{Z}\times\mathcal{Y}\rightarrow\mathcal{X},
\end{equation}
where $z\sim p_Z$ is a latent variable. The generator induces the class-conditional distribution
\begin{equation}
P_{G,y}
=
\bigl(G_{\theta}(\cdot,y)\bigr)_{\#}p_Z,
\end{equation}
where $(\cdot)_{\#}$ denotes the pushforward measure. Training aims to approximate the true class-conditional distributions,
\begin{equation}
P_{G,y}
\approx
P_y,
\qquad
y\in\mathcal{Y}.
\end{equation}

%-----------------------------------------------------------------------------------------------------------------
\subsection{Wasserstein Approximation Error}
%-----------------------------------------------------------------------------------------------------------------

To quantify the discrepancy between real and generated distributions, we employ the Wasserstein--1 distance
\begin{equation}
W_1(\mu,\nu)
=
\inf_{\gamma\in\Pi(\mu,\nu)}
\mathbb{E}_{(x,\tilde{x})\sim\gamma}
\left[\|x-\tilde{x}\|\right],
\end{equation}
where $\Pi(\mu,\nu)$ denotes the set of all couplings with marginals $\mu$ and $\nu$. Equivalently, by Kantorovich--Rubinstein duality,
\begin{equation}
W_1(\mu,\nu)
=
\sup_{\|g\|_{\mathrm{Lip}}\le 1}
\left|
\mathbb{E}_{\mu}[g(x)]
-
\mathbb{E}_{\nu}[g(x)]
\right|.
\end{equation}
For each class $y$, we define the class-conditional approximation error as
\begin{equation}
\varepsilon_y
=
W_1(P_y,P_{G,y}).
\end{equation}

Smaller values of $\varepsilon_y$ indicate closer agreement between the real and generated class-conditional distributions and therefore higher generative fidelity.

%-----------------------------------------------------------------------------------------------------------------
\subsection{Synthetic Augmentation as Distribution Mixing}
%-----------------------------------------------------------------------------------------------------------------

After training, synthetic samples are incorporated into the data through the mixture distribution
\begin{equation}
\label{augmented distribution}
P_{\alpha}
=
(1-\alpha)P
+
\alpha P_G,
\end{equation}
where $\alpha\in[0,1]$ denotes the augmentation strength and
\begin{equation}
P_G(x,y)
=
\pi_y P_{G,y}(x),
\qquad
\pi_y=P(Y=y).
\end{equation}
The corresponding augmented risk is
\begin{equation}
\label{augmented risk}
R_{P_{\alpha}}(f)
=
\mathbb{E}_{(x,y)\sim P_{\alpha}}
\left[\ell(f(x),y)\right].
\end{equation}

The following theorem establishes how the class-conditional approximation errors $\{\varepsilon_y\}$ influence the difference between the original and augmented risks.

%%-----------------------------------------------------------------------------------------------------------------
\subsection{Risk Control Under Conditional Wasserstein Approximation}
%%-----------------------------------------------------------------------------------------------------------------
The following theorem establishes a Wasserstein-based bound on the risk distortion induced by synthetic augmentation. It shows that the difference between the original and augmented risks scales linearly with the augmentation strength $\alpha$ and the class-conditional approximation errors $\{\varepsilon_y\}_{y=1}^K$. As a result, higher-fidelity generators yield tighter control of augmentation-induced risk distortion.

%%-------------------------------------------------------------------THEOREM----------------------------------------------
\begin{theorem}[Risk Stability Under Conditional Generative Augmentation] Assume:
\label{Theorem1}
\begin{enumerate}
    \item The loss $\ell(f(x),y)$ is $L_\ell$-Lipschitz with respect to $x$.
    
    \item The hypothesis class $\mathcal F$ is measurable and induces a uniformly bounded loss.
    
    \item For each class $y\in\mathcal Y$,
    \[
    W_1(P_y,P_{G,y})
    \le
    \varepsilon_y .
    \]
\end{enumerate}

Then for any classifier $f$,
\begin{equation}
|R_{P_\alpha}(f) - R_P(f)|
\le
\alpha L_\ell \sum_{y \in \mathcal{Y}} \pi_y \varepsilon_y.
\end{equation}

If $\varepsilon_y \le \varepsilon$ uniformly, then
\begin{equation}
|R_{P_\alpha}(f) - R_P(f)|
\le
\alpha L_\ell \varepsilon.
\end{equation}
\end{theorem}\ \\
These assumptions are standard in optimal transport and statistical learning theory. Assumption 1 links Wasserstein discrepancies to differences in expected loss, Assumption 2 ensures the existence of the relevant expectations, and Assumption 3 quantifies the class-conditional approximation error of the generator.

\begin{proof}[Proof]

Let $P$ denote the true joint distribution and $P_G$ the synthetic joint distribution defined as in Equation \ref{augmented distribution}. By linearity of expectation, we decompose the augmented risk (\ref{augmented risk}) under mixture distributions.
\begin{align*}
R_{P_\alpha}(f)
&=
\mathbb{E}_{(x,y)\sim P_\alpha}
[\ell(f(x), y)] \\
&=
(1-\alpha)\mathbb{E}_{P}[\ell(f(x),y)]
+
\alpha \mathbb{E}_{P_G}[\ell(f(x),y)].
\end{align*}
Hence,
\begin{align*}
R_{P_\alpha}(f) - R_P(f)
&=
(1-\alpha)R_P(f)
+
\alpha R_{P_G}(f)
-
R_P(f) \\
&=
\alpha\big(R_{P_G}(f) - R_P(f)\big).
\end{align*}
Taking absolute values,
\begin{equation*}
|R_{P_\alpha}(f) - R_P(f)|
=
\alpha |R_{P_G}(f) - R_P(f)|.
\end{equation*}
Thus it suffices to bound $|R_{P_G}(f) - R_P(f)|$. Next by using the law of total expectation, we do the conditional decomposition as
\begin{align*}
R_{P_G}(f)
&=
\sum_{y \in \mathcal{Y}}
\pi_y
\mathbb{E}_{x \sim P_{G,y}}
[\ell(f(x), y)],
\\
R_P(f)
&=
\sum_{y \in \mathcal{Y}}
\pi_y
\mathbb{E}_{x \sim P_y}
[\ell(f(x), y)].
\end{align*}
Therefore,
\begin{align}
R_{P_G}(f) - R_P(f)
=
\sum_{y \in \mathcal{Y}}
\pi_y
\left(
\mathbb{E}_{P_{G,y}}[\ell(f(x), y)]
-
\mathbb{E}_{P_y}[\ell(f(x), y)]
\right).
\end{align}
Taking absolute values and applying the triangle inequality,
\begin{align}
\label{diffeq}
|R_{P_G}(f) - R_P(f)|
\le
\sum_{y \in \mathcal{Y}}
\pi_y
\left|
\mathbb{E}_{P_{G,y}}[\ell(f(x), y)]
-
\mathbb{E}_{P_y}[\ell(f(x), y)]
\right|.
\end{align}
By assumption, the loss $\ell(f(x), y)$ is $L_\ell$-Lipschitz in $x$.  Define $g_y(x) := \ell(f(x), y)$. Then
\[
\|g_y\|_{\text{Lip}} \le L_\ell.
\]

Let $\tilde g_y(x) := \frac{1}{L_\ell} g_y(x)$. Then $\|\tilde g_y\|_{\text{Lip}} \le 1$. By Kantorovich--Rubinstein duality,
\[
\left|
\mathbb{E}_{P_{G,y}}[\tilde g_y(x)]
-
\mathbb{E}_{P_y}[\tilde g_y(x)]
\right|
\le
W_1(P_y, P_{G,y}).
\]
Multiplying both sides by $L_\ell$ gives
\begin{align}\label{mull}
\left|
\mathbb{E}_{P_{G,y}}[\ell(f(x), y)]
-
\mathbb{E}_{P_y}[\ell(f(x), y)]
\right|
\le
L_\ell W_1(P_y, P_{G,y}).
\end{align}
Substituting \ref{mull} into \ref{diffeq},
\begin{align}
|R_{P_G}(f) - R_P(f)|
\le
L_\ell
\sum_{y \in \mathcal{Y}}
\pi_y
W_1(P_y, P_{G,y}).
\end{align}
By assumption, $ W_1(P_y, P_{G,y}) \le \varepsilon_y$.  Thus,
\begin{align}
|R_{P_G}(f) - R_P(f)|
\le
L_\ell
\sum_{y \in \mathcal{Y}}
\pi_y \varepsilon_y.
\end{align}
Finally,
\begin{align}
|R_{P_\alpha}(f) - R_P(f)|
\le
\alpha L_\ell
\sum_{y \in \mathcal{Y}}
\pi_y \varepsilon_y.
\end{align}
If $\varepsilon_y \le \varepsilon$ uniformly,
\[
|R_{P_\alpha}(f) - R_P(f)|
\le
\alpha L_\ell \varepsilon.
\]
This completes the proof.
\end{proof}

In practice, the class-conditional Wasserstein discrepancies provide a quantitative measure of generative fidelity. Theorem \ref{Theorem1} therefore establishes a direct connection between distributional fidelity and augmentation reliability: smaller Wasserstein discrepancies imply tighter control of augmentation-induced risk distortion.

%%----------------------------------------------------------------------------------------------------------
\subsection{Asymptotic Consistency of Conditional Generative Augmentation}
%%-------------------------------------------------------------------THEOREM----------------------------------------------
While risk bounds under finite Wasserstein discrepancy quantify the immediate impact of generative augmentation, it is also important to understand the long-term behavior as the generator improves. Specifically, if the generator is able to produce samples that increasingly approximate the true class-conditional distributions, we expect the induced risk distortion to vanish as discussed in \cite{BenDavid2010}.

%%-------------------------------------------------------------------THEOREM----------------------------------------------
\begin{corollary}[Consistency Under Vanishing Transport Error]
\label{Theorem2}
Suppose $\{G_n\}$ is a sequence of generators satisfying
\begin{equation}
\max_{y \in \mathcal{Y}}
W_1(P_y, P_{G_n,y})
\to 0.
\end{equation}
If $\alpha_n \to \alpha \in [0,1]$, then for any fixed classifier $f$,
\begin{equation}
R_{P_{\alpha_n}}(f) \to R_P(f).
\end{equation}
\end{corollary}

This result establishes that conditional generative augmentation is asymptotically unbiased: as the generator becomes perfectly accurate in the Wasserstein sense, the expected classification risk under the augmented distribution converges to that under the true distribution. In other words, improvements in generator quality translate directly into more reliable augmentation.

\begin{proof}[Proof ]

Let $\{G_n\}$ be a sequence of generators such that
$
\max_{y \in \mathcal{Y}}
W_1(P_y, P_{G_n,y})
\to 0.
$ From Theorem \ref{Theorem1},
\[
|R_{P_{\alpha_n}}(f) - R_P(f)|
\le
\alpha_n L_\ell
\max_{y}
W_1(P_y, P_{G_n,y}).
\]

Since $0 \le \alpha_n \le 1$, the sequence $\{\alpha_n\}$ is bounded.  Because $\max_{y}
W_1(P_y, P_{G_n,y}) \to 0,$ the product converges to zero. Hence,
\[
R_{P_{\alpha_n}}(f) \to R_P(f).
\]

This establishes asymptotic unbiasedness under vanishing Wasserstein approximation error.
\end{proof}

These findings are consistent with prior results in statistical learning theory and optimal transport, which establish that excess risk vanishes as the Wasserstein discrepancy between training and target distributions converges to zero, yielding asymptotic consistency of learning under distributional approximation \cite{Redko2017}.

%-------------------------------------------------------------------------------------------------------------
\subsection{Capacity-Dependent Generalization via Rademacher Complexity}
%-------------------------------------------------------------------------------------------------------------

The Wasserstein-based bounds developed above quantify the bias introduced by generative augmentation at the population level. In practice, however, classifiers are learned from finite samples and selected from a hypothesis class $\mathcal{F}$. Consequently, generalization performance depends not only on distributional approximation error but also on the capacity of the hypothesis class.

To capture this effect, we employ Rademacher complexity, a standard measure of function class capacity in statistical learning theory. Let $\sigma_1,\ldots,\sigma_n$ be independent Rademacher random variables taking values $\pm1$ with equal probability. For the loss-composed class

\[
\ell \circ \mathcal{F}
=
\left\{
(x,y)\mapsto \ell(f(x),y)
:\,
f\in\mathcal{F}
\right\},
\]

the empirical Rademacher complexity is defined as

\[
\widehat{\mathfrak{R}}_n(\ell\circ\mathcal{F})
=
\mathbb{E}_{\sigma}
\left[
\sup_{f\in\mathcal{F}}
\frac1n
\sum_{i=1}^{n}
\sigma_i
\ell(f(x_i),y_i)
\right].
\]

Smaller values of $\widehat{\mathfrak{R}}_n(\ell\circ\mathcal{F})$ indicate a lower-capacity hypothesis class and stronger generalization guarantees. When training on the augmented distribution \ref{augmented distribution}, the generalization error consists of two components:
\begin{enumerate}
\item \textbf{Estimation error}, controlled by the Rademacher complexity of $\mathcal{F}$.
\item \textbf{Augmentation bias}, controlled by the Wasserstein approximation errors $\{\varepsilon_y\}$.
\end{enumerate}

Next we combine these quantities to derive a finite-sample generalization bound that explicitly accounts for hypothesis class capacity, augmentation strength, and generative fidelity.
%%---------------------------------------------------------------------------------
\subsection{Capacity-Dependent Excess Risk Under Generative Augmentation}
%%-------------------------------------------------------------------------------------------------------------
While previous results quantified the bias induced by generative augmentation via class-conditional Wasserstein distances, and asymptotic consistency guaranteed vanishing bias for perfect generators, in practice we work with finite samples and a restricted hypothesis class. The generalization ability of the learned classifier therefore depends not only on the fidelity of the synthetic data but also on the capacity of the function class used for learning.

To formalize this, we analyze the empirical risk minimizer trained on a mixture of real and synthetic data. By combining Wasserstein-based bias control with Rademacher complexity bounds that capture hypothesis class capacity, we obtain a high-probability bound on the excess risk. This bound explicitly quantifies how the risk depends on:

\begin{enumerate}
\item The augmentation strength $\alpha$ and class-conditional approximation errors $\epsilon_y$ (bias due to the generator).
\item The complexity of the hypothesis class (variance due to finite samples).
\item The sample size $N$ and desired confidence $1-\delta$.
\end{enumerate}

The following theorem formalizes these trade-offs, showing that with high probability, the empirical minimizer achieves an excess risk that balances the fidelity of the generative augmentation with the expressive power of the classifier.
%%-------------------------------------------------------------------THEOREM----------------------------------------------
%%-------------------------------------------------------------------THEOREM----------------------------------------------
\begin{theorem}[Capacity-Dependent Generalization Bound]
\label{Theorem3}

Let $\varepsilon
=
\max_{y\in\mathcal Y}
W_1(P_y,P_{G,y})$ denote the maximum class-conditional Wasserstein approximation error. Under the assumptions of Theorem~\ref{Theorem1}, for any $\delta\in(0,1)$, with probability at least $1-\delta$,

\begin{equation}
R_P(\widehat f_\alpha)
\le
\inf_{f\in\mathcal F}R_P(f)
+
2\mathfrak R_n(\ell\circ\mathcal F)
+
3B\sqrt{\frac{\log(2/\delta)}{2n}}
+
\alpha L_\ell \varepsilon.
\end{equation}

\end{theorem}

The proof follows from standard Rademacher complexity
generalization bounds \cite{BartlettMendelson2002,Mohri2018}
combined with Theorem~\ref{Theorem1}.

\begin{proof}

We decompose excess risk into estimation error and augmentation bias. For any $f \in \mathcal{F}$,
\begin{align*}
R_P(\widehat f_\alpha)
&=
R_{P_\alpha}(\widehat f_\alpha)
+
\big(R_P(\widehat f_\alpha) - R_{P_\alpha}(\widehat f_\alpha)\big).
\end{align*}
By Theorem \ref{Theorem1}, $|R_P(f) - R_{P_\alpha}(f)|
\le
\alpha L_\ell \varepsilon. $ Thus we get, $
R_P(\widehat f_\alpha)
\le
R_{P_\alpha}(\widehat f_\alpha)
+
\alpha L_\ell \varepsilon$. Applying the standard Rademacher generalization bound, with probability at least $1-\delta$,
\begin{align}\label{radbound}
R_{P_\alpha}(\widehat f_\alpha)
\le
\inf_{f \in \mathcal{F}} R_{P_\alpha}(f)
+
2 \mathfrak R_n(\ell \circ \mathcal{F})
+
3B \sqrt{\frac{\log(2/\delta)}{2n}}.
\end{align}
Again using Theorem \ref{Theorem1} and considering the infimum over $f \in \mathcal{F}$ gives,
\begin{align}
\inf_{f \in \mathcal{F}} R_{P_\alpha}(f)
\le
\inf_{f \in \mathcal{F}} R_P(f)
+
\alpha L_\ell \varepsilon.
\end{align}
Substituting this back into Equation \ref{radbound},
\begin{align}
R_P(\widehat f_\alpha)
\le
\inf_{f \in \mathcal{F}} R_P(f)
+
2 \mathfrak R_n(\ell \circ \mathcal{F})
+
3B \sqrt{\frac{\log(2/\delta)}{2n}}
+
\alpha L_\ell \varepsilon.
\end{align}

\end{proof}

The bound highlights a fundamental trade-off between estimation error and augmentation bias. While lower Wasserstein discrepancies yield tighter control of augmentation-induced risk distortion, predictive performance also depends on hypothesis class capacity and finite-sample effects. Consequently, improvements in generative fidelity do not necessarily translate into proportional gains in classification accuracy.

%-----------------------------------------------------------------------------------------------------------------
\section{Experimental Design}
%-----------------------------------------------------------------------------------------------------------------

This section describes the datasets, augmentation procedures, evaluation metrics, and experimental settings used to validate the proposed theoretical framework. The experiments are designed to assess both the distributional fidelity of generated samples and their impact on downstream classification performance under class imbalance.

%-----------------------------------------------------------------------------------------------------------------
\subsection{Experimental Objectives}
%-----------------------------------------------------------------------------------------------------------------

The empirical study evaluates the reliability of conditional generative augmentation by assessing distributional fidelity, classification performance under varying augmentation strengths, differences between CGAN and CWGAN-GP generated samples, and the relationship between generative fidelity and downstream predictive performance.

%-----------------------------------------------------------------------------------------------------------------
\subsection{Datasets}
%-----------------------------------------------------------------------------------------------------------------

To evaluate the proposed framework across different imbalance settings, experiments were conducted on two benchmark classification datasets exhibiting natural class imbalance.

\begin{table}[H]
\centering
\caption{Summary of datasets used in the empirical evaluation.}
\label{tab:dataset_summary}
\begin{tabular}{lcccc}
\hline
Dataset & Samples & Features & Classes & Minority Class (\%) \\
\hline
Credit Card Default & 30,000 & 23 & 2 & 22.1 \\
Forest Cover Type & 581,012 & 54 & 7 & 0.47 \\
\hline
\end{tabular}
\end{table}

The Credit Card Default dataset is a binary classification problem in which the objective is to predict whether a customer will default on the next month's payment. The dataset contains financial and demographic attributes and exhibits a naturally imbalanced class distribution, making it a widely used benchmark for studying imbalance-aware learning methods.

The Forest Cover Type dataset is a multiclass classification problem consisting of cartographic and environmental variables used to predict forest cover categories. The dataset contains seven classes with substantially different class frequencies, providing a challenging large-scale benchmark for evaluating class-conditional generative augmentation in multiclass settings.

Together, these datasets enable evaluation of the proposed framework across both binary and multiclass imbalance scenarios while providing substantially different data characteristics and class distributions.

%-----------------------------------------------------------------------------------------------------------------
\subsection{Experimental Setup}
%-----------------------------------------------------------------------------------------------------------------

\paragraph{Generative Augmentation Models.}
Two conditional generative models are considered in this study: the CGAN and the CWGAN-GP. Both models are trained to learn class conditional feature distributions from the training data and subsequently generate synthetic samples for data augmentation. The generated samples are then combined with the original training data according to the augmentation strength parameter $\alpha$ to construct the augmented datasets used for downstream classification.\\

\noindent We compare CWGAN-GP against the following baselines:
\begin{itemize}
    \item \textbf{Real}: Training on the original imbalanced dataset without augmentation.
    \item \textbf{ROS}: Random Oversampling, replicating minority-class samples until class balance is achieved. While computationally efficient and simple to implement, this method often leads to overfitting due to duplicate data points.
    \item \textbf{SMOTE}: Synthetic Minority Over-sampling Technique based on nearest-neighbor interpolation.
    \item \textbf{CGAN}: Conditional GAN-based augmentation without Wasserstein regularization.
\end{itemize}

\paragraph{Downstream Classifiers.}
Downstream classification performance is evaluated using two representative classifiers: a Random Forest (RF) with 100 trees and a Multi-Layer Perceptron (MLP) with two hidden layers and ReLU activations. These models were selected to represent two distinct learning paradigms, allowing the effects of synthetic augmentation to be assessed across different classifier families.

\paragraph{Augmentation Strength $\alpha$.}
To assess the effect of augmentation intensity, experiments were conducted with augmentation strengths $\alpha \in \{0.25, 0.50, 0.75, 1.00\}$.

\paragraph{Evaluation Metrics and Statistical Testing.}
Performance is evaluated using Macro-F1 score and Macro Recall, which are appropriate for assessing classification performance under class imbalance. Distributional fidelity is quantified using both the Average Class-Wise Wasserstein Distance and the Weighted Class-Wise Wasserstein Distance between real and generated class-conditional distributions.

\paragraph{Implementation Details.}
CGAN and CWGAN-GP were trained using dataset-specific hyperparameter settings selected through preliminary experimentation. For the Credit Card dataset, a latent dimension of 64 and 500 training epochs were used, whereas the Forest Cover dataset employed a latent dimension of 32 and 300 training epochs. In all experiments, models were optimized using the Adam optimizer with a learning rate of $10^{-4}$, and CWGAN-GP was trained using a gradient penalty coefficient of 10. All reported results were averaged over three independent random seeds. Complete hyperparameter settings are provided in Appendix \ref{hype}.

%-----------------------------------------------------------------------------------------------------------------
\section{Results and Discussion}
%-----------------------------------------------------------------------------------------------------------------

%-----------------------------------------------------------------------------------------------------------------
\subsection{Distributional Fidelity Analysis}
%-----------------------------------------------------------------------------------------------------------------
We first evaluate the fidelity of the generated samples using Wasserstein-based measures. Since the theoretical framework developed in Section~3 identifies Wasserstein approximation error as a key determinant of augmentation reliability, these results provide a direct empirical assessment of the proposed risk bounds.

For the Credit Card dataset, fidelity was measured using the Wasserstein distance between the real and generated minority-class distributions. For the Forest Cover dataset, fidelity was assessed using both the Average Class-Wise Wasserstein Distance and the Weighted Class-Wise Wasserstein Distance.

\begin{table}[H]
\centering
\caption{
Distributional fidelity measured using Wasserstein distances (mean $\pm$ standard deviation over three random seeds).
}
\label{wasserstein_results}
\begin{tabular}{llcc}
\toprule
Dataset & Method & Wasserstein Distance & Average Class-Wise $W1$ \\
\midrule
Credit Card & CGAN & $27408.45 \pm 567.19$ & -- \\
Credit Card & CWGAN-GP & $3714.99 \pm 559.21$ & -- \\
\midrule
Forest Cover & CGAN & $0.1231 \pm 0.0158$ & $0.0601 \pm 0.0072$ \\
Forest Cover & CWGAN-GP & $0.1072 \pm 0.0164$ & $0.0543 \pm 0.0027$ \\
\bottomrule
\end{tabular}
\end{table}

Table~\ref{wasserstein_results} shows that CWGAN-GP consistently achieves lower Wasserstein discrepancies than CGAN across both datasets. On the Credit Card dataset, the Wasserstein distance decreases from approximately $27,408$ to $3,715$, while on the Forest Cover dataset CWGAN-GP achieves lower average and weighted class-wise Wasserstein distances. These results indicate that CWGAN-GP produces synthetic samples that more closely approximate the underlying class-conditional distributions.

Because Wasserstein discrepancy appears explicitly in the theoretical risk bounds of Section~3, lower values correspond to tighter control of augmentation-induced risk distortion. Thus, the observed reductions in Wasserstein distance provide empirical support for the proposed framework and suggest that CWGAN-GP generates higher-fidelity synthetic data than CGAN.

For the Credit Card dataset, Wasserstein discrepancy increased with augmentation strength for both generative methods. However, CWGAN-GP consistently maintained substantially lower Wasserstein values than CGAN, ranging from approximately 6,340 to 11,744 compared with 89,924 to 166,518 for CGAN.

The next subsection investigates whether these improvements in distributional fidelity translate into improved downstream classification performance.

%-----------------------------------------------------------------------------------------------------------------
\subsection{Classification Performance Under Augmentation}
%-----------------------------------------------------------------------------------------------------------------

We next evaluate the impact of generative augmentation on downstream classification performance. Following the experimental protocol described in Section~4, performance is assessed using Macro-F1 score and Macro Recall, with results averaged over three independent random seeds. These metrics provide complementary perspectives on classification quality under class imbalance, with Macro-F1 measuring overall predictive performance and Macro Recall emphasizing performance on underrepresented classes.

\begin{table}[H]
\centering
\caption{
Classification performance on the Credit Card dataset (mean $\pm$ standard deviation over three random seeds).
}
\label{cc_performance}
\begin{tabular}{lccc}
\toprule
Method & $\alpha$ & Macro-F1 & Macro Recall \\
\midrule
Real & 0.00 &
$0.6720 \pm 0.0108$ &
$0.6495 \pm 0.0104$ \\
\midrule
ROS & 0.50 &
$\mathbf{0.6809 \pm 0.0088}$ &
$\mathbf{0.6714 \pm 0.0148}$ \\
\midrule
SMOTE & 0.50 &
$0.6778 \pm 0.0138$ &
$0.6699 \pm 0.0061$ \\
\midrule
CGAN & 0.25 &
$0.6728 \pm 0.0092$ &
$0.6494 \pm 0.0094$ \\
 & 0.50 &
$0.6751 \pm 0.0049$ &
$0.6520 \pm 0.0057$ \\
 & 0.75 &
$0.6689 \pm 0.0208$ &
$0.6471 \pm 0.0190$ \\
 & 1.00 &
$0.6764 \pm 0.0087$ &
$0.6537 \pm 0.0085$ \\
\midrule
CWGAN-GP & 0.25 &
$0.6721 \pm 0.0126$ &
$0.6493 \pm 0.0124$ \\
& 0.50 &
$0.6742 \pm 0.0094$ &
$0.6523 \pm 0.0080$ \\
& 0.75 &
$0.6742 \pm 0.0128$ &
$0.6525 \pm 0.0121$ \\
 & 1.00 &
$0.6714 \pm 0.0152$ &
$0.6497 \pm 0.0131$ \\
\bottomrule
\end{tabular}
\end{table}

\begin{table}[H]
\centering
\caption{
Classification performance on the Forest Cover dataset (mean $\pm$ standard deviation over three random seeds).
}
\label{forest_performance}
\begin{tabular}{lccc}
\toprule
Method & $\alpha$ & Macro-F1 & Macro Recall \\
\midrule
Real & 0.00 &
$0.7727 \pm 0.0209$ &
$0.7407 \pm 0.0049$ \\
\midrule
ROS& 0.25 &
$0.7887 \pm 0.0327$ &
$0.8055 \pm 0.0193$ \\
\midrule
SMOTE & 0.50 &
$\mathbf{0.7901 \pm 0.0397}$ &
$\mathbf{0.8183 \pm 0.0235}$ \\
\midrule
CGAN & 0.25 &
$0.7545 \pm 0.0439$ &
$0.7185 \pm 0.0343$ \\
 & 0.50 &
$0.7514 \pm 0.0448$ &
$0.7121 \pm 0.0375$ \\
& 0.75 &
$0.7452 \pm 0.0467$ &
$0.7128 \pm 0.0307$ \\
 & 1.00 &
$0.7490 \pm 0.0476$ &
$0.7151 \pm 0.0328$ \\
\midrule
CWGAN-GP & 0.25 &
$0.7413 \pm 0.0553$ &
$0.7006 \pm 0.0476$ \\
 & 0.50 &
$0.7295 \pm 0.0677$ &
$0.6960 \pm 0.0531$ \\
 & 0.75 &
$0.7279 \pm 0.0679$ &
$0.6886 \pm 0.0628$ \\
 & 1.00 &
$0.7236 \pm 0.0764$ &
$0.6883 \pm 0.0639$ \\
\bottomrule
\end{tabular}
\end{table}

Tables~\ref{cc_performance} and~\ref{forest_performance} summarize the classification results for the Credit Card and Forest Cover datasets, respectively. For ROS and SMOTE, we report the best-performing augmentation level across the same candidate levels used for the generative methods. For CGAN and CWGAN-GP, results are reported separately for each value of $\alpha$ to evaluate sensitivity to augmentation strength. To improve readability, only the Random Forest (RF) and Multi-Layer Perceptron (MLP) classifiers are reported in the main text, while additional classifier results are provided in the Appendix \ref{classifier-tables}.

For the Credit Card dataset, all augmentation methods achieve performance close to the real-data baseline. Random Oversampling obtains the highest Macro-F1 and Macro Recall, followed closely by SMOTE. Among the generative methods, CGAN and CWGAN-GP produce comparable classification performance across augmentation strengths, despite the much lower Wasserstein discrepancy achieved by CWGAN-GP. This suggests that improved distributional fidelity does not necessarily translate into large predictive gains in the binary setting.

For the Forest Cover dataset, SMOTE achieves the strongest overall performance, with a Macro-F1 of $0.7901$ and Macro Recall of $0.8183$. Random Oversampling is also highly competitive, while the real-data baseline remains stronger than both generative methods. CGAN outperforms CWGAN-GP in predictive performance across all augmentation levels, even though CWGAN-GP has lower Wasserstein discrepancy. This contrast provides important evidence that distributional fidelity and downstream predictive utility are related but not equivalent.

To assess whether observed performance differences are statistically significant, paired Wilcoxon signed-rank tests were conducted across the three random seeds. No statistically significant differences were observed between CGAN and CWGAN-GP in Macro-F1 performance on either dataset ($p > 0.05$), despite the substantially lower Wasserstein discrepancies achieved by CWGAN-GP.

Taken together, the results highlight an important distinction between distributional fidelity and predictive utility. Although the previous subsection demonstrated that CWGAN-GP consistently achieves lower Wasserstein discrepancies than CGAN, these improvements do not universally translate into superior classification performance. In particular, the Forest Cover experiments show that classical oversampling methods remain highly competitive from a predictive perspective. This observation motivates the more detailed minority-class analysis presented next and provides an important empirical context for interpreting the proposed Wasserstein-based risk bounds.\\

\begin{figure}[H]
\begin{subfigure}{0.48\textwidth}
\centering
\includegraphics[width=\textwidth]{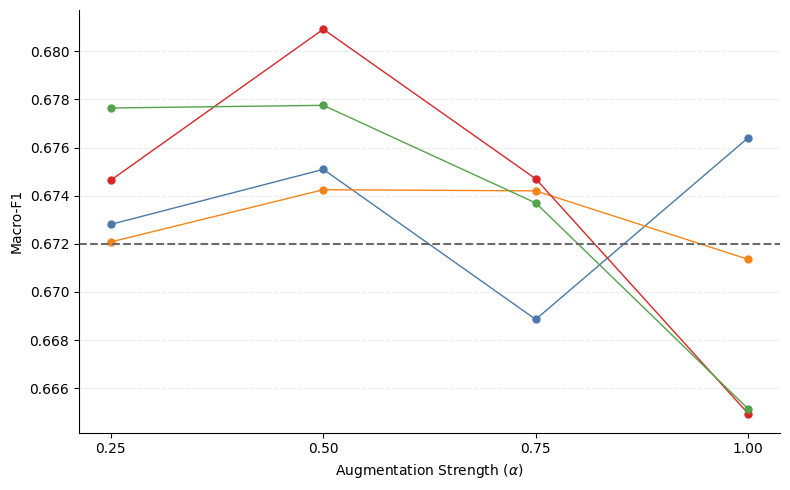}
\caption{Credit Card}
\end{subfigure}
\hfill
\begin{subfigure}{0.48\textwidth}
\centering
\includegraphics[width=\textwidth]{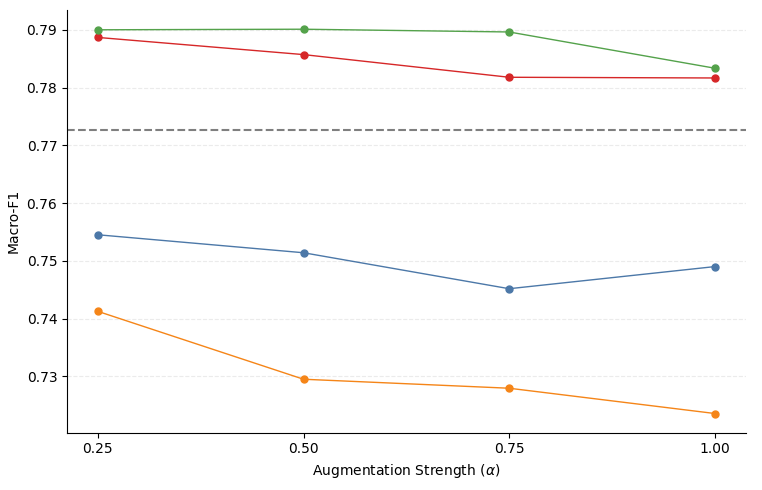}
\caption{Forest Cover}
\end{subfigure}
\begin{center}
\footnotesize
\textcolor[HTML]{6B6B6B}{\rule[0.5ex]{1.5em}{1.5pt}}
\ Real
\qquad
\textcolor[HTML]{4C78A8}{\rule[0.5ex]{1.5em}{1.5pt}}
\ CGAN
\qquad
\textcolor[HTML]{F58518}{\rule[0.5ex]{1.5em}{1.5pt}}
\ CWGAN-GP
\qquad
\textcolor[HTML]{54A24B}{\rule[0.5ex]{1.5em}{1.5pt}}
\ SMOTE
\qquad
\textcolor[HTML]{D62728}{\rule[0.5ex]{1.5em}{1.5pt}}
\ ROS
\end{center}
\centering
\caption{
Macro-F1 as a function of augmentation strength $\alpha$ for the Credit Card and Forest Cover datasets.
}
\label{alpha_f1}
\end{figure}

Figure~\ref{alpha_f1} provides a visual summary of the effect of augmentation strength on classification performance. For both CGAN and CWGAN-GP, increasing the proportion of synthetic samples generally leads to a gradual reduction in Macro-F1. In contrast, the performance of SMOTE and Random Oversampling remains relatively stable across augmentation strengths. These results suggest that excessive reliance on synthetic data may adversely affect predictive performance despite improvements in distributional fidelity.

%-----------------------------------------------------------------------------------------------------------------
\subsection{Minority-Class Analysis}
%-----------------------------------------------------------------------------------------------------------------

While aggregate metrics such as Macro-F1 and Macro Recall provide an overall assessment of classification performance, they may obscure behavior on underrepresented classes. To better understand the effect of augmentation on rare categories, we examine class-specific recall for the four least frequent classes in the Forest Cover dataset: Aspen, Cottonwood/Willow, Douglas Fir, and Krummholz. 

Since the Credit Card dataset involves only two classes, Macro Recall already provides a direct measure of minority-class detection performance, and a separate class-wise analysis is therefore omitted.

\begin{table}[H]
\centering
\caption{
Minority-class recall for the Forest Cover dataset (mean $\pm$ standard deviation over three random seeds). Results are reported using the best-performing augmentation strength for each method: $\alpha=0.25$ for CGAN and CWGAN-GP, $\alpha=0.50$ for SMOTE, and $\alpha=0.00$ for the real-data baseline.
}
\label{tab:minority_recall}
\begin{tabular}{lcccc}
\toprule
Class & Real & CGAN & CWGAN-GP & SMOTE \\
\midrule
Aspen &
$0.371 \pm 0.020$ &
$0.371 \pm 0.039$ &
$0.369 \pm 0.025$ &
$\mathbf{0.641 \pm 0.022}$ \\

Cottonwood/Willow &
$0.718 \pm 0.051$ &
$0.742 \pm 0.049$ &
$0.723 \pm 0.057$ &
$\mathbf{0.822 \pm 0.035}$ \\

Douglas Fir &
$0.625 \pm 0.026$ &
$0.628 \pm 0.029$ &
$0.630 \pm 0.033$ &
$\mathbf{0.759 \pm 0.010}$ \\

Krummholz &
$0.803 \pm 0.012$ &
$0.811 \pm 0.016$ &
$0.811 \pm 0.021$ &
$\mathbf{0.878 \pm 0.016}$ \\
\bottomrule
\end{tabular}
\end{table}

The results indicate that CGAN and CWGAN-GP provide only modest improvements over the real-data baseline. Small gains are observed for Cottonwood/Willow, Douglas Fir, and Krummholz, while recall for Aspen remains largely unchanged. In contrast, SMOTE consistently achieves the highest recall across all four minority classes, with particularly large improvements for Aspen and Douglas Fir. These improvements reflect SMOTE's ability to directly increase minority-class representation through interpolation-based oversampling.

\begin{figure}[H]
\centering
\includegraphics[width=0.75\textwidth]{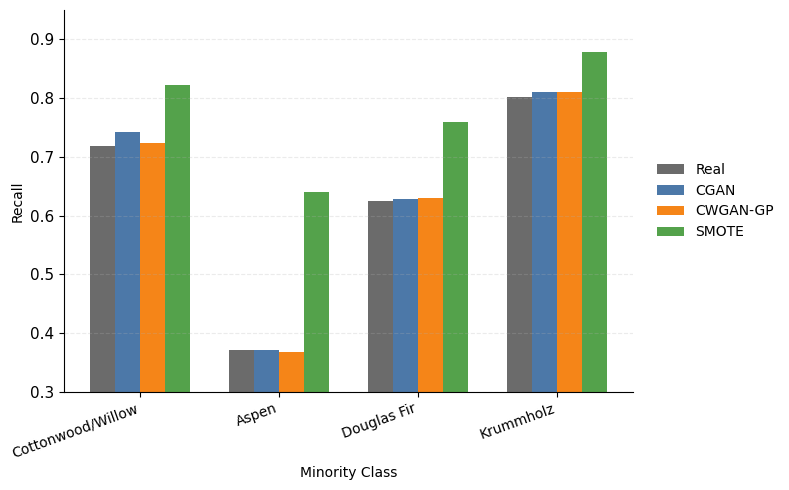}
\caption{
Recall of the four least frequent classes in the Forest Cover dataset using the Random Forest classifier. Results are reported for the best-performing augmentation strength for each method ($\alpha=0.25$ for CGAN and CWGAN-GP, $\alpha=0.50$ for SMOTE, and $\alpha=0$ for the real-data baseline).}
\label{minority_recall}
\end{figure}

Interestingly, the recall values obtained by CGAN and CWGAN-GP are often very similar despite the substantial differences in their Wasserstein fidelity reported in Table~\ref{wasserstein_results}. For example, CWGAN-GP achieves considerably lower Wasserstein discrepancies than CGAN, yet both methods produce nearly identical recall for Douglas Fir and Krummholz. This observation suggests that improvements in distributional fidelity do not necessarily translate into proportional gains in minority-class predictive performance.

Figure~\ref{minority_recall} provides a visual summary of these results. Although CWGAN-GP achieves substantially lower Wasserstein discrepancies than CGAN, the resulting minority-class recall values remain comparable across most classes. This finding reinforces a central conclusion of the study: distributional fidelity and predictive utility are related but distinct aspects of augmentation quality. While Wasserstein-based generative models produce higher-fidelity synthetic samples, improvements in distributional approximation do not necessarily translate into proportional gains in downstream classification performance.

\subsection{Empirical Validation of the Theory}

The empirical results support the theoretical role of Wasserstein discrepancy in the proposed risk bounds. Table~\ref{wasserstein_results} shows that CWGAN-GP yields smaller Wasserstein approximation errors than CGAN on both datasets, implying a smaller augmentation-bias term in Theorem~\ref{Theorem1}. 
The dependence of Wasserstein discrepancy on augmentation strength further supports Theorem~\ref{Theorem1}. For the Credit Card dataset, the Wasserstein distance increased from 89,924 to 166,518 under CGAN as $\alpha$ increased from 0.25 to 1.00. CWGAN-GP exhibited the same trend but with substantially smaller discrepancies, ranging from 6,340 to 11,744. Since the theoretical risk distortion term scales as $\alpha \varepsilon$, these results suggest that stronger augmentation magnifies the impact of distributional approximation error.
However, Tables~\ref{cc_performance} and~\ref{forest_performance} show that lower Wasserstein discrepancy does not necessarily produce the highest Macro-F1 or Macro Recall. This behavior is consistent with Theorem~\ref{Theorem3}, where predictive performance depends not only on augmentation bias but also on hypothesis class capacity and finite-sample estimation effects.

The empirical findings therefore provide partial validation of the proposed theory. The Wasserstein approximation error successfully explains differences in distributional fidelity between CGAN and CWGAN-GP and behaves consistently with the theoretical risk bounds. However, predictive performance is also influenced by classifier capacity, finite-sample effects, and the interaction between synthetic and real observations. These findings support the interpretation of Wasserstein discrepancy as a measure of augmentation reliability rather than a direct predictor of classification accuracy.

\section{Discussion}

This study investigated conditional generative augmentation from both theoretical and empirical perspectives. The proposed framework establishes a direct connection between generative fidelity and classification risk through Wasserstein-based approximation error, providing a principled basis for evaluating the reliability of synthetic data augmentation. Unlike many existing studies that assess augmentation methods primarily through predictive performance, the present work emphasizes distributional fidelity as a fundamental property of synthetic data quality.

The empirical results demonstrate that CWGAN-GP consistently achieves lower Wasserstein discrepancies than CGAN across both the Credit Card and Forest Cover datasets. These findings are consistent with the theoretical analysis, which identifies the class-conditional Wasserstein distance as a key quantity governing augmentation-induced risk distortion. From a distributional perspective, the results indicate that Wasserstein-based adversarial training produces synthetic samples that more closely approximate the underlying data-generating process.

At the same time, the experiments reveal that improvements in distributional fidelity do not necessarily translate into superior classification performance. In several settings, classical oversampling methods such as SMOTE achieved higher Macro-F1 scores and minority-class recall than either generative augmentation approach. This observation highlights an important distinction between synthetic data fidelity and predictive utility. While fidelity measures how accurately a generator reproduces the underlying distribution, classification performance depends additionally on factors such as class separability, decision-boundary complexity, sample size, and the inductive bias of the downstream learning algorithm.

An interesting observation is that SMOTE frequently achieved higher classification performance than both CGAN and CWGAN-GP, particularly on the Forest Cover dataset. Because the Forest Cover dataset contains a relatively large number of minority-class observations, interpolation-based oversampling may be sufficient to improve decision-boundary estimation without requiring accurate modeling of the full class-conditional distribution. By contrast, generative augmentation seeks to approximate the entire class-conditional distribution, a substantially more difficult task that may not yield proportional gains in predictive performance when the primary challenge is class imbalance rather than distributional sparsity. These results suggest that higher-fidelity synthetic data does not necessarily imply superior classification accuracy and that the effectiveness of augmentation depends on both the data characteristics and the downstream learning objective.

The results further suggest that augmentation strength should be viewed as a data-dependent parameter rather than a universally optimal quantity. For the Forest Cover dataset, increasing the proportion of synthetic data generally led to reduced predictive performance, whereas the Credit Card dataset exhibited modest improvements under stronger CWGAN-GP augmentation. These contrasting behaviors indicate that the effectiveness of augmentation depends on the interaction between dataset characteristics, class imbalance structure, and generative fidelity.

Taken together, the findings support the central premise of the proposed framework: generative augmentation should be evaluated not only through predictive metrics but also through distributional reliability. Wasserstein-based fidelity provides a principled measure of how closely synthetic data approximates the underlying class-conditional distributions and therefore offers a theoretically grounded perspective on augmentation quality. More broadly, the study suggests that distributional fidelity and predictive performance should be regarded as complementary evaluation criteria rather than interchangeable measures of success.

\paragraph{Limitations} 
Several limitations should be acknowledged. First, the empirical evaluation is limited to two structured tabular datasets; additional studies across a broader range of domains would strengthen the generality of the conclusions. Second, the theoretical analysis relies on standard Lipschitz continuity assumptions and Wasserstein-based distributional comparisons, which may not fully capture the behavior of highly complex models in high-dimensional settings. Third, only CGAN and CWGAN-GP architectures were considered, while newer generative models such as diffusion and transformer-based approaches may exhibit different fidelity–performance relationships. Finally, the proposed framework provides guarantees on distributional fidelity and augmentation-induced risk distortion rather than direct guarantees of classification improvement.

\section{Conclusion}

This work developed a statistical framework for analyzing conditional generative augmentation in imbalanced classification. By modeling augmentation as a distribution-mixing process, we derived Wasserstein-based risk bounds that explicitly connect generative fidelity, augmentation strength, and classification risk. The resulting theory provides a principled foundation for understanding how distributional approximation error influences the reliability of synthetic data augmentation.

Empirical evaluations on the Credit Card Default and Forest Cover datasets support the theoretical framework. Across both datasets, CWGAN-GP consistently achieved lower Wasserstein discrepancies than CGAN, indicating closer approximation of the underlying class-conditional distributions. At the same time, the results showed that improved distributional fidelity does not necessarily yield superior classification performance, as classical oversampling methods such as SMOTE often remained competitive or outperformed generative approaches. These findings highlight that distributional reliability and predictive utility are related but distinct aspects of synthetic data quality.

Overall, the proposed framework extends the evaluation of synthetic augmentation beyond predictive accuracy alone by introducing distributional reliability as a theoretically grounded criterion. Rather than viewing augmentation solely as a tool for improving classification performance, this perspective provides a principled approach for assessing the quality and trustworthiness of synthetic data in machine learning applications.

\section{Future Work}
A particularly promising direction is the development of adaptive augmentation strategies in which augmentation strength is selected on a class-by-class basis. The empirical results suggest that a fixed augmentation level may not be optimal across all classes or datasets. Future work will investigate augmentation schemes that incorporate class imbalance, generative fidelity, and structural characteristics of the data to determine class-specific augmentation strengths with accompanying theoretical guarantees.

\section*{Code Availability}

The code used to generate all experimental results, figures, and tables will be made publicly available upon acceptance.

\bibliography{References.bib}

\begin{thebibliography}{10}

\bibitem{HeGarcia2009}
H.~He and E.~A. Garcia, ``Learning from imbalanced data,'' {\em IEEE
  Transactions on Knowledge and Data Engineering}, vol.~21, no.~9,
  pp.~1263--1284, 2009.

\bibitem{Lucic2018}
M.~Lucic, K.~Kurach, M.~Michalski, S.~Gelly, and O.~Bousquet, ``Are gans
  created equal? a large-scale study,'' in {\em Advances in Neural Information
  Processing Systems}, vol.~31, pp.~700--709, 2018.

\bibitem{Goodfellow2014}
I.~Goodfellow, J.~Pouget-Abadie, M.~Mirza, B.~Xu, D.~Warde-Farley, S.~Ozair,
  A.~Courville, and Y.~Bengio, ``Generative adversarial nets,'' in {\em
  Advances in Neural Information Processing Systems}, vol.~27, pp.~2672--2680,
  Curran Associates, Inc., 2014.

\bibitem{Arjovsky2017}
M.~Arjovsky, S.~Chintala, and L.~Bottou, ``Wasserstein generative adversarial
  networks,'' in {\em Proceedings of the 34th International Conference on
  Machine Learning} (D.~Precup and Y.~W. Teh, eds.), vol.~70 of {\em
  Proceedings of Machine Learning Research}, pp.~214--223, PMLR, 2017.

\bibitem{Gulrajani2017}
I.~Gulrajani, F.~Ahmed, M.~Arjovsky, V.~Dumoulin, and A.~Courville, ``Improved
  training of wasserstein gans,'' in {\em Advances in Neural Information
  Processing Systems}, vol.~30, pp.~5767--5777, 2017.

\bibitem{DouzasBacao2018}
G.~Douzas and F.~Bacao, ``Effective data generation for imbalanced learning
  using generative adversarial networks,'' {\em Information Sciences},
  vol.~465, pp.~1--20, 2018.

\bibitem{Chawla2002}
N.~V. Chawla, K.~W. Bowyer, L.~O. Hall, and W.~P. Kegelmeyer, ``Smote:
  Synthetic minority over-sampling technique,'' {\em Journal of Artificial
  Intelligence Research}, vol.~16, pp.~321--357, 2002.

\bibitem{Mirza2014}
M.~Mirza and S.~Osindero, ``Conditional generative adversarial nets,'' in {\em
  Proceedings of the Workshop on Challenges in Representation Learning, ICML},
  2014.

\bibitem{FridAdar2018}
M.~Frid-Adar, I.~Diamant, E.~Klang, M.~Amitai, J.~Goldberger, and H.~Greenspan,
  ``Gan-based synthetic medical image augmentation for increased cnn
  performance,'' {\em Neurocomputing}, vol.~321, pp.~321--331, 2018.

\bibitem{BenDavid2010}
S.~Ben-David, J.~Blitzer, K.~Crammer, A.~Kulesza, F.~Pereira, and J.~Vaughan,
  ``A theory of learning from different domains,'' {\em Machine Learning},
  vol.~79, no.~1--2, pp.~151--175, 2010.

\bibitem{Redko2017}
I.~Redko, A.~Habrard, M.~Sebban, and R.~Flamary, ``Optimal transport for
  multi-source domain adaptation,'' {\em IEEE Transactions on Pattern Analysis
  and Machine Intelligence}, vol.~41, no.~8, pp.~1853--1865, 2017.

\bibitem{BartlettMendelson2002}
P.~L. Bartlett and S.~Mendelson, ``Rademacher and gaussian complexities: Risk
  bounds and structural results,'' {\em Journal of Machine Learning Research},
  vol.~3, pp.~463--482, 2002.

\bibitem{Mohri2018}
M.~Mohri, A.~Rostamizadeh, and A.~Talwalkar, {\em Foundations of Machine
  Learning}.
\newblock MIT Press, 2~ed., 2018.

\end{thebibliography}

\end{document}

\maketitle
\setcounter{section}{0}
\renewcommand{\theequation}{A\arabic{equation}}
\renewcommand{\thesection}{A}
\vspace{-.75in}

 %================================================
\appendix
%================================================

 %-----------------------------------------------------------------------------------------------------------------
\section{ Classifier-Specific Performance Results}\label{classifier-tables}
%-----------------------------------------------------------------------------------------------------------------

This appendix reports classifier-specific results for the Random Forest (RF) and Multi-Layer Perceptron (MLP) models. While the main text presents aggregated performance summaries, the tables below provide a more detailed view of how individual classifiers respond to different augmentation strategies and augmentation strengths.

%-----------------------------------------------------------------------------------------------------------------
\subsection*{A.1 Credit Card Dataset}
%-----------------------------------------------------------------------------------------------------------------

\begin{table}[H]
\centering
\caption{Credit Card dataset: classifier-specific Macro-F1 scores (mean $\pm$ standard deviation over three random seeds).}
\label{tab:credit_classifier_results}
\begin{tabular}{llcc}
\toprule
Method & $\alpha$ & RF & MLP \\
\midrule

Real & 0.00 &
$0.6798 \pm 0.0052$ &
$0.6642 \pm 0.0090$ \\

\midrule
CGAN & 0.25 &
$0.6801 \pm 0.0042$ &
$0.6655 \pm 0.0056$ \\
& 0.50 &
$0.6774 \pm 0.0049$ &
$0.6728 \pm 0.0046$ \\
& 0.75 &
$0.6793 \pm 0.0038$ &
$0.6584 \pm 0.0272$ \\
& 1.00 &
$0.6790 \pm 0.0061$ &
$0.6738 \pm 0.0114$ \\

\midrule
CWGAN-GP & 0.25 &
$0.6803 \pm 0.0060$ &
$0.6639 \pm 0.0127$ \\
& 0.50 &
$0.6798 \pm 0.0074$ &
$0.6687 \pm 0.0086$ \\
& 0.75 &
$0.6803 \pm 0.0073$ &
$0.6681 \pm 0.0157$ \\
& 1.00 &
$0.6832 \pm 0.0062$ &
$0.6595 \pm 0.0110$ \\

\midrule
ROS & 0.25 &
$0.6743 \pm 0.0058$ &
$0.6750 \pm 0.0092$ \\
& 0.50 &
$0.6819 \pm 0.0053$ &
$0.6800 \pm 0.0128$ \\
& 0.75 &
$0.6864 \pm 0.0061$ &
$0.6630 \pm 0.0053$ \\
& 1.00 &
$0.6877 \pm 0.0050$ &
$0.6422 \pm 0.0178$ \\

\midrule
SMOTE & 0.25 &
$0.6856 \pm 0.0053$ &
$0.6697 \pm 0.0114$ \\
& 0.50 &
$0.6899 \pm 0.0059$ &
$0.6656 \pm 0.0003$ \\
& 0.75 &
$0.6882 \pm 0.0021$ &
$0.6592 \pm 0.0039$ \\
& 1.00 &
$0.6874 \pm 0.0079$ &
$0.6429 \pm 0.0111$ \\

\bottomrule
\end{tabular}
\end{table}
Table~\ref{tab:credit_classifier_results} presents classifier-specific Macro-F1 scores for the Credit Card dataset. Similar to the Forest Cover results, Random Forest consistently achieves higher performance than MLP across most augmentation settings. The improvements obtained through Random Oversampling and SMOTE are particularly pronounced for RF, suggesting that the classifier benefits from increased minority-class representation. In contrast, the performance differences between CGAN and CWGAN-GP remain relatively small despite the substantial reductions in Wasserstein discrepancy achieved by CWGAN-GP. This pattern is observed for both RF and MLP, indicating that improvements in generative fidelity do not necessarily produce proportional gains in predictive performance. Importantly, the overall ordering of augmentation methods remains stable across classifiers, providing additional evidence that the conclusions drawn from the aggregated results are not driven by a particular modeling choice.

%-----------------------------------------------------------------------------------------------------------------
\subsection*{A.2 Forest Cover Type Dataset}
%-----------------------------------------------------------------------------------------------------------------

\begin{table}[H]
\centering
\caption{Forest Cover dataset: classifier-specific Macro-F1 scores (mean $\pm$ standard deviation over three random seeds).}
\label{tab:forest_classifier_results}
\begin{tabular}{llcc}
\toprule
Method & $\alpha$ & RF & MLP \\
\midrule

Real & 0.00 &
$0.7912 \pm 0.0031$ &
$0.7542 \pm 0.0077$ \\

\midrule
CGAN & 0.25 &
$0.7929 \pm 0.0049$ &
$0.7161 \pm 0.0193$ \\
& 0.50 &
$0.7918 \pm 0.0055$ &
$0.7110 \pm 0.0088$ \\
& 0.75 &
$0.7871 \pm 0.0086$ &
$0.7033 \pm 0.0101$ \\
& 1.00 &
$0.7921 \pm 0.0041$ &
$0.7059 \pm 0.0094$ \\

\midrule
CWGAN-GP & 0.25 &
$0.7915 \pm 0.0041$ &
$0.6910 \pm 0.0066$ \\
& 0.50 &
$0.7907 \pm 0.0031$ &
$0.6683 \pm 0.0136$ \\
& 0.75 &
$0.7877 \pm 0.0087$ &
$0.6682 \pm 0.0276$ \\
& 1.00 &
$0.7921 \pm 0.0020$ &
$0.6550 \pm 0.0222$ \\

\midrule
ROS & 0.25 &
$0.8172 \pm 0.0091$ &
$0.7601 \pm 0.0121$ \\
& 0.50 &
$0.8164 \pm 0.0113$ &
$0.7551 \pm 0.0028$ \\
& 0.75 &
$0.8159 \pm 0.0084$ &
$0.7477 \pm 0.0098$ \\
& 1.00 &
$0.8166 \pm 0.0064$ &
$0.7467 \pm 0.0153$ \\

\midrule
SMOTE & 0.25 &
$0.8214 \pm 0.0051$ &
$0.7586 \pm 0.0163$ \\
& 0.50 &
$0.8251 \pm 0.0118$ &
$0.7551 \pm 0.0113$ \\
& 0.75 &
$0.8194 \pm 0.0086$ &
$0.7599 \pm 0.0131$ \\
& 1.00 &
$0.8177 \pm 0.0101$ &
$0.7490 \pm 0.0081$ \\

\bottomrule
\end{tabular}
\end{table}
Table~\ref{tab:forest_classifier_results} reports classifier-specific Macro-F1 scores for the Forest Cover dataset. Several observations emerge. First, Random Forest consistently outperforms MLP across all augmentation strategies, indicating that tree-based ensemble methods are particularly effective for modeling the complex nonlinear structure of the Forest Cover data. Second, the relative ranking of augmentation methods remains largely unchanged across classifiers. In both RF and MLP, Random Oversampling and SMOTE achieve the strongest performance, while CGAN and CWGAN-GP generally remain below the real-data baseline. Third, although CWGAN-GP achieves substantially lower Wasserstein discrepancies than CGAN, these improvements do not translate into superior classifier performance for either RF or MLP. This finding further supports the central conclusion of the paper that distributional fidelity and predictive utility capture distinct aspects of augmentation quality. Overall, the classifier-specific results demonstrate that the main conclusions reported in the aggregated analysis are robust across different learning algorithms.

%-------------------------------------------------------------
\section{Experimental Hyperparameters}\label{hype}
%-------------------------------------------------------------

Tables~\ref{tab:hyperparameters} summarize the hyperparameters used for all experiments. Unless otherwise stated, identical classifier settings were used across augmentation methods to ensure fair comparisons.

\begin{table}[H]
\centering
\caption{Hyperparameters used for the Credit Card and Forest Cover experiments.}
\label{tab:hyperparameters}
\begin{tabular}{lcc}
\toprule
Parameter & Credit Card & Forest Cover \\
\midrule
Number of Runs & 3 & 3 \\
Augmentation Strengths ($\alpha$) &
${0.25,0.50,0.75,1.00}$ &
${0.25,0.50,0.75,1.00}$ \\

Train/Test Split &
Stratified 80\% / 20\% &
Stratified 80\% / 20\% \\

Latent Dimension & 64 & 32 \\
Batch Size & 256 & 64 \\
Training Epochs & 300 & 300 \\
Optimizer & Adam & Adam \\
Learning Rate & $10^{-4}$ & $10^{-4}$ \\

\midrule
\multicolumn{3}{c}{\textbf{CWGAN-GP}} \\
\midrule

Critic Updates per Generator Step & 3 & 5 \\
Gradient Penalty Coefficient ($\lambda_{GP}$) & 10 & 10 \\

\midrule
\multicolumn{3}{c}{\textbf{Random Forest}} \\
\midrule

Number of Trees & 100 & 100 \\
Criterion & Gini & Gini \\

\midrule
\multicolumn{3}{c}{\textbf{MLP}} \\
\midrule

Hidden Layers & (128, 64) & (128, 64) \\
Activation Function & ReLU & ReLU \\
Maximum Iterations & 2000 & 2000 \\

\bottomrule
\end{tabular}
\end{table}

\section*{Code Availability}

The code used to generate all experimental results and figures reported in this study is publicly available at:

\url{https://github.com/yourusername/generative-augmentation-reliability}

The repository contains implementations of CGAN, CWGAN-GP, evaluation metrics, and scripts for reproducing all experiments.